\documentclass[11pt]{article}

\usepackage[final]{acl}

\usepackage{times}
\usepackage{latexsym}

\usepackage[T1]{fontenc}
\usepackage[utf8]{inputenc}

\usepackage{microtype}

\usepackage{inconsolata}

\usepackage{graphicx}

\usepackage[utf8]{inputenc}
\usepackage[T1]{fontenc}
\usepackage{amsmath,amssymb,amsthm}
\usepackage{mathtools}
\usepackage{hyperref}
\usepackage{cleveref}
\usepackage{booktabs} 
\theoremstyle{plain}
\newtheorem{theorem}{Theorem}[section]

\newtheorem{proposition}[theorem]{Proposition}

\theoremstyle{definition}

\theoremstyle{remark}

\usepackage{amsmath} 
\title{The Cost of Compression: A Rate-Distortion Limit on Factual Hallucination}

\newcommand{\samethanks}{\footnotemark[\value{footnote}]}

\author{
  \textbf{Xi Wang}\textsuperscript{1,2}\thanks{Equal contribution.}
  \quad
  \textbf{Shijia Xu}\textsuperscript{3}\samethanks
  \quad
  \textbf{Rongfeng Guo}\textsuperscript{4}
  \\
  \textsuperscript{1}Hefei Institutes of Physical Science, Chinese Academy of Sciences, China
  \\
  \textsuperscript{2}University of Science and Technology of China, China
  \\
  \textsuperscript{3}Chongqing University, China
  \\
  \textsuperscript{4}Shenzhen University, China
  \\
  \texttt{xw\_cs@mail.ustc.edu.cn}
  \quad
  \texttt{shijiaxu@stu.cqu.edu.cn}
}

\begin{document}
\maketitle

\begin{abstract}
Factual hallucination in closed-book question answering is often treated as a coverage problem: a model fails because the relevant fact is absent from its internal memory. 
This view misses a second source of error. 
Even when a fact has been observed, finite memory may force it to be stored only approximately. 
We study this effect through a simple \emph{coverage--compression} model of factual recall.

We consider an unstructured question-answering task with $N$ possible queries and $K$ possible answers. 
A learner observes $M$ training facts, compresses them into at most $B$ bits, and answers uniformly drawn test queries without retrieval. 
For a uniformly random ground-truth mapping, we prove
\[
\mathcal{E}
\ge
\frac{M}{N}\,
\delta^\star\!\left(\frac{B}{M}\right)
+
\left(1-\frac{M}{N}\right)
\left(1-\frac{1}{K}\right),
\]
where $\delta^\star(r)$ is the inverse rate-distortion function of a uniform $K$-ary source under zero-one loss. 
The two terms separate \emph{compression distortion} on observed facts from \emph{missing coverage} on unobserved facts.

The bound gives a compact way to reason about selective memory, forced compression, structure, retrieval, abstention, and long-context organization. 
We study the predicted signatures with theory-implied simulations and controlled fact-injection probes in modern language models that vary fact load and effective trainable memory. 
The result is not a complete theory of hallucination, but an information-theoretic account of a separable failure mode: lossy recall of observed facts under finite memory.
\end{abstract}

\section{Introduction}
\label{sec:intro}

Large language models can answer many factual questions without explicit retrieval, suggesting that some factual knowledge is stored in their parameters. 
Yet factual errors remain common across short-form, long-form, knowledge-graph, dialogue-level, and in-the-wild factuality evaluations \citep{wei2024simpleqa,wei2024longformfactuality,liu2024grapheval,luo2024halludial,zhao2024wildhallucinations,bayat2024factbench,jacovi2025factsgrounding,liu2025verifact}. 
A standard explanation is missing coverage: the relevant fact was absent from the model's training data, inaccessible to its parameters, or unavailable at inference time. 
This explanation motivates retrieval and grounding methods, which move part of the factual burden from parametric memory to external evidence \citep{yan2024crag,edge2024graphrag,wang2024ragbestpractices,ru2024ragchecker,lee2025finetunerag,wu2025multirag}. 
It is important, but incomplete.

A fact can be observed and still be recalled incorrectly. 
A compressed database gives the basic intuition: if a record is absent, the system must guess; if the record is present but stored in a lossy code, the system may retrieve the wrong value. 
Closed-book factual recall faces the same tension. 
A model must compress many facts into a finite internal representation, and recent controlled studies of new factual knowledge, knowledge editing, memorization, and knowledge capacity suggest that factual storage is itself a constrained resource \citep{gekhman2024newknowledge,huang2024hallueditbench,allenzhu2024knowledgecapacity,morris2025memorize,zhou2025taskstratifiedknowledge}. 
The resulting error is not merely a failure of exposure, but a failure of recall under compression.

\begin{figure*}[t]
    \centering
    \includegraphics[width=1.0\textwidth]{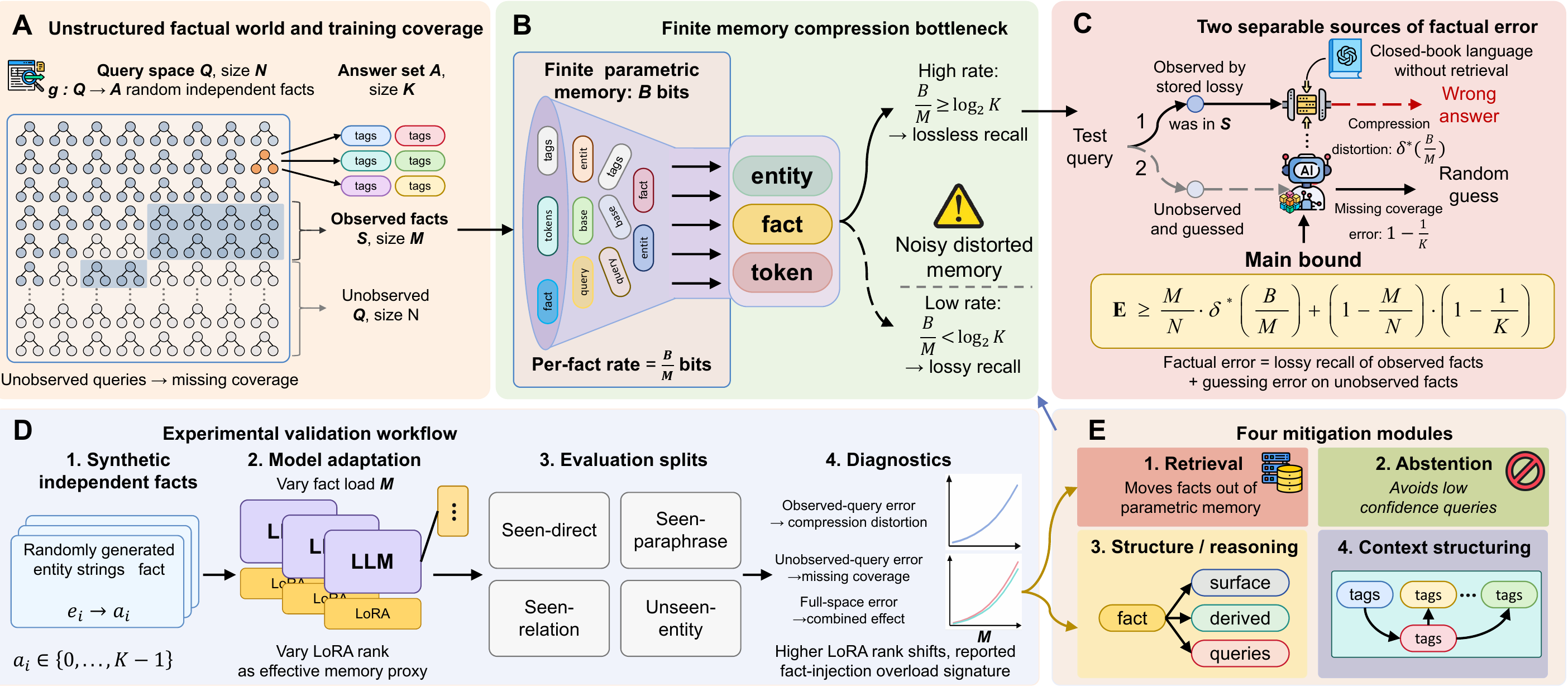}
    \caption{
    Coverage--compression framework. 
    Closed-book factual error decomposes into missing coverage on unobserved facts and compression distortion on observed facts stored under finite memory. 
    The same diagnostics are used in controlled fact-injection experiments: observed-query error measures lossy recall, unobserved-query error measures missing coverage, and full-space error measures their mixture.
    }
    \label{fig:framework}
\end{figure*}

We isolate this mechanism with a deliberately simple model, summarized in Figure~\ref{fig:framework}. 
Let $\mathcal{Q}$ be a query set of size $N$ and $\mathcal{A}$ an answer set of size $K$. 
The ground-truth mapping $g:\mathcal{Q}\to\mathcal{A}$ is drawn uniformly at random, so the answers contain no semantic regularity that a learner can exploit. 
A learner observes $M$ training pairs, encodes the observed information into at most $B$ bits, and answers uniformly drawn test queries without retrieval. 
This random-source setting is not a model of all factual knowledge; it is a null model for unstructured factual memory, where correct prediction must come from stored information. 
For example, when $K=10$, one lossless answer requires $\log_2 10\approx 3.32$ bits, so a $160$-bit answer memory stores only about $48$ independent facts under the optimistic free-addressing convention.

The central claim is that closed-book factual error has two separable sources: the queried fact may not be effectively covered by memory, or it may be covered but stored too lossily to be recalled correctly. 
Conditioned on the training set, the observed labels are $M$ independent $K$-ary symbols. 
The learner's memory is a compressed representation of this sequence. 
For a uniform $K$-ary source under zero-one loss, let $\delta^\star(r)$ denote the minimum achievable reconstruction error when each symbol is encoded at rate $r$ bits. 
Our main result, stated formally in Theorem~\ref{thm:rd}, gives
\[
\mathcal{E}
\ge
\frac{M}{N}\,
\delta^\star\!\left(\frac{B}{M}\right)
+
\left(1-\frac{M}{N}\right)
\left(1-\frac{1}{K}\right).
\]
The first term is the distortion on observed facts; the second is the guessing error on unobserved queries. 
When $B/M\ge \log_2K$, observed facts can be stored losslessly and the bound reduces to the usual unseen-query floor. 
When $B/M<\log_2K$, even facts seen during training have a nonzero recall error floor.

This decomposition clarifies several behaviors that are often discussed separately. 
Finite memory alone does not imply that more data must hurt an optimal learner: with selective allocation, additional observations can simply produce diminishing returns. 
Compression overload appears when many observed facts are forced to share a fixed effective representation, lowering the per-fact rate $B/M$. 
Structure changes the trade-off by reducing the number of independent facts that must be stored. 
Long-context organization gives an analogous working-memory problem, where facts are present in the prompt but compete for limited effective attention; recent long-context evaluations show that retrieval becomes substantially harder when context length, task complexity, latent structure, long in-context learning, or non-literal matching increases \citep{hsieh2024ruler,zhang2024infinitebench,bai2024longbenchv2,li2024needlebench,vodrahalli2024michelangelo,li2024longiclbench,modarressi2025nolima}. 

We evaluate the framework through predicted signatures rather than aggregate benchmark scores. 
Theory-implied simulations instantiate the exact distortion-rate envelope, while controlled fact-injection probes vary the number of independent synthetic facts and the effective trainable memory budget in modern language models. 
Across both components, we report observed-query error for compression distortion, unobserved-query error for missing coverage, and full-space error for their mixture.

Taken together, the paper makes three contributions. 
First, it identifies a compression-induced source of factual error that is separate from missing coverage. 
Second, it proves a rate-distortion lower bound showing how observed-fact distortion and unseen-query guessing combine in a single coverage--compression trade-off. 
Third, it turns the bound into a diagnostic framework, yielding measurable signatures in theory-implied simulations and controlled fact-injection probes. 
The claim is intentionally narrow: compression is not a complete theory of hallucination, but a separable mechanism that predicts when observed facts become unreliable under finite effective memory.

\section{Problem Setting}
\label{sec:setting}

We study closed-book factual recall under a finite memory budget. 
Let $\mathcal{Q}$ be a query set with $|\mathcal{Q}|=N$, and let $\mathcal{A}$ be an answer set with $|\mathcal{A}|=K\ge 2$. 
The ground-truth mapping is
\[
g:\mathcal{Q}\to\mathcal{A}.
\]
To isolate unstructured factual memory, we draw $g$ uniformly at random from all $K^N$ mappings. 
Equivalently, the labels $\{g(q):q\in\mathcal{Q}\}$ are independent uniform $K$-ary symbols. 
This random-source model is not a claim that real knowledge is random; it is a null model in which semantic, logical, and distributional shortcuts have been removed.

A learner observes a uniformly sampled training set $S\subset\mathcal{Q}$ of size $M$, together with the labels
\[
X_S=(g(q))_{q\in S}.
\]
It encodes the observed information into a memory state $W\in\mathcal{W}$ with $|\mathcal{W}|\le 2^B$. 
At test time, it receives a query $q$ and must answer without retrieval. 
We allow arbitrary randomized encoders and decoders. 
Under the free-addressing convention, the decoder may condition on the training index set $S$ and the query identity:
\[
\begin{aligned}
\psi &: \mathcal{Q}\times\mathcal{S}_M\times\mathcal{W}
\to \Delta(\mathcal{A}),\\
\mathcal{S}_M
&=
\{S\subseteq\mathcal{Q}: |S|=M\}.
\end{aligned}
\]
Let $\hat g(q)$ denote the answer sampled from, or deterministically chosen by, this decoder. 
The budget $B$ is charged only to stored answer information; query identities and the set $S$ are treated as available addresses. 
This convention favors the learner. 
If query identities must also be stored, an additional address cost on the order of $\log_2 \binom{N}{M}$ is required.

We measure expected full-space error under the uniform test distribution:
\begin{equation}
\label{eq:risk}
\mathcal{E}
=
\mathbb{E}
\left[
\frac{1}{N}
\sum_{q\in\mathcal{Q}}
\mathbf{1}\{\hat g(q)\neq g(q)\}
\right],
\end{equation}
where the expectation is over the random ground truth, the sampled training set, and learner randomness.

We distinguish two notions of coverage. 
A query has \emph{training coverage} if $q\in S$. 
It has \emph{effective memory coverage} if the learner has allocated enough information to recall its answer with nontrivial fidelity. 
A fact may therefore be present in the training set but still be recalled incorrectly.

\section{Coverage--Compression Theory}
\label{sec:theory}

The theory separates two sources of error. 
If a test query was not observed, its label is independent of the learner's memory and the learner must guess. 
If the query was observed, the learner still has to reconstruct its label from a finite-memory representation. 
The first effect gives a coverage floor; the second is a rate-distortion problem.

\paragraph{Coverage-only floor.}
If every observed fact were recalled perfectly, the only unavoidable error would come from queries outside $S$:
\begin{equation}
\label{eq:coverage}
\mathcal{E}
\ge
\left(1-\frac{M}{N}\right)
\left(1-\frac{1}{K}\right).
\end{equation}
For $q\notin S$, the label $g(q)$ is uniform over $\mathcal{A}$ and independent of $(S,X_S,W)$, so no predictor can be correct with probability greater than $1/K$. 
This baseline treats observation as equivalent to lossless storage. 
We next remove that assumption.

\paragraph{Observed facts as source coding.}
Conditioned on $S$, the observed labels $X_S$ are $M$ independent uniform symbols over a $K$-ary alphabet. 
The memory state $W$ is a compressed representation of this sequence, and answering an observed query reconstructs one coordinate of $X_S$ from the compressed memory and the query address. 
The relevant distortion is factual recall error,
\[
d(a,\hat a)=\mathbf{1}\{a\neq \hat a\}.
\]

For a uniform $K$-ary source under zero-one distortion, the rate-distortion function is
\begin{equation}
\label{eq:rd-function}
\begin{aligned}
R_K(D)
&=
\log_2 K
-
h_2(D)
-
D\log_2(K-1),\\
&\hspace{2.0em}
0\le D\le 1-\frac{1}{K},
\end{aligned}
\end{equation}
where $h_2(D)=-D\log_2D-(1-D)\log_2(1-D)$. 
This is the standard rate-distortion function for a uniform $K$-ary source under Hamming distortion \citep{shannon1959coding,cover2006elements}. 
Thus perfect recall requires $\log_2 K$ bits per fact, while random guessing has distortion $1-1/K$ and requires zero bits. 
Let
\begin{equation}
\label{eq:delta}
\delta^\star(r)
=
\inf\{D\ge0:R_K(D)\le r\}
\end{equation}
denote the smallest possible reconstruction error at rate $r$ bits per observed fact.

\begin{theorem}[Coverage--compression bound]
\label{thm:rd}
Under the random-source model and the free-addressing convention, every learner with memory budget $B$ satisfies
\begin{equation}
\label{eq:rd-bound}
\begin{aligned}
\mathcal{E}
&\ge
\frac{M}{N}
\delta^\star\!\left(\frac{B}{M}\right)\\
&\quad+
\left(1-\frac{M}{N}\right)
\left(1-\frac{1}{K}\right).
\end{aligned}
\end{equation}
\end{theorem}

\begin{proof}[Proof idea]
Condition on $S$. 
Since $W$ has at most $2^B$ values,
\[
I(X_S;W\mid S)\le H(W\mid S)\le B.
\]
Let $\widehat X_S=(\hat g(q))_{q\in S}$ denote the decoder's reconstructions on the observed queries. 
Because $\widehat X_S$ is produced from $(S,W)$ and decoder randomness, data processing gives
\[
I(X_S;\widehat X_S\mid S)\le B.
\]
If the observed-fact reconstruction error is $D_{\mathrm{obs}}$, the rate-distortion converse for the uniform $K$-ary source implies
\[
R_K(D_{\mathrm{obs}})\le \frac{B}{M},
\qquad
D_{\mathrm{obs}}
\ge
\delta^\star\!\left(\frac{B}{M}\right).
\]
A uniformly drawn query lies in $S$ with probability $M/N$ and outside $S$ with probability $1-M/N$. 
The former contributes the observed-fact distortion term; the latter contributes the guessing error $1-1/K$. 
Combining the two cases gives Eq.~\eqref{eq:rd-bound}. 
The full proof is in Appendix~\ref{app:rd-proof}.
\end{proof}

The bound decomposes the error into
\[
\begin{aligned}
\mathcal{E}_{\mathrm{obs}}
&=
\frac{M}{N}\delta^\star(B/M),\\
\mathcal{E}_{\mathrm{miss}}
&=
\left(1-\frac{M}{N}\right)(1-1/K).
\end{aligned}
\]
The first term is lossy recall of observed facts; the second is guessing on unobserved queries.
When $B/M\ge\log_2K$, observed facts can be stored losslessly and the first term vanishes. 
When $B/M<\log_2K$, even facts that appeared in training have a nonzero recall error floor.

\begin{proposition}[Tightness under free addressing]
\label{prop:tightness}
Under the same free-addressing convention, the bound in Theorem~\ref{thm:rd} is achievable up to finite-blocklength effects. 
Specifically, if $M\to\infty$ with $B/M\to r$, then standard rate-distortion codes achieve observed-fact distortion $\delta^\star(r)+o(1)$, while unobserved queries are answered by random guessing.
\end{proposition}

This shows that Eq.~\eqref{eq:rd-bound} is not merely a loose lower bound: in the random-source model, it is the natural rate-distortion envelope. 
The proof and finite-blocklength caveat are given in Appendix~\ref{app:rd-proof}.

\paragraph{Address-charged variant.}
The free-addressing convention gives the learner the training index set $S$ and query identities as addresses. 
If those identities must be encoded explicitly, storing which $M$ queries are covered requires an additional address cost on the order of $\log_2\binom{N}{M}$ bits. 
A conservative address-charged analogue replaces the answer-memory budget by
\[
B_{\mathrm{ans}}
=
\left[
B-\log_2\binom{N}{M}
\right]_+,
\]
where $[x]_+=\max\{x,0\}$. 
The same rate-distortion argument then gives
\[
\mathcal{E}
\ge
\frac{M}{N}
\delta^\star\!\left(\frac{B_{\mathrm{ans}}}{M}\right)
+
\left(1-\frac{M}{N}\right)
\left(1-\frac{1}{K}\right).
\]
We use the free-addressing convention in the main theorem because it isolates answer-memory compression and favors the learner; the address-charged variant quantifies how explicit indexing would only make the memory constraint tighter.

\paragraph{Consequences.}
The lossless endpoint gives the capacity scale
\[
L_{\mathrm{lossless}}
\approx
\frac{B}{\log_2K},
\]
the number of independent $K$-ary facts that can be stored without error under free addressing. 
This is a capacity scale, not an optimal training size.

Finite memory alone does not imply that more observations must hurt an optimal learner. 
With selective allocation, the learner may spend memory on a subset of the facts and ignore the rest. 
An idealized selective-memory curve is
\begin{equation*}
\begin{aligned}
\mathcal{E}_{\mathrm{sel}}(M)
&=
\min\left\{
1-\frac{1}{K},
\min_{1\le L\le M}\Phi(L)
\right\}, \\
\Phi(L)
&=
\frac{L}{N}\delta^\star\!\left(\frac{B}{L}\right)
+
\left(1-\frac{L}{N}\right)
\left(1-\frac{1}{K}\right).
\end{aligned}
\end{equation*}
Because the feasible set grows with $M$, $\mathcal{E}_{\mathrm{sel}}(M)$ is non-increasing in $M$. 
Selective memory therefore predicts saturation or diminishing returns, not unavoidable degradation.

Compression overload appears under the stronger regime in which all $M$ observed facts must share one fixed effective representation. 
Then the operating point is Eq.~\eqref{eq:rd-bound}: increasing $M$ improves training coverage but lowers the per-fact rate $B/M$, so observed-fact fidelity can degrade even when full-space error does not. 
A U-shaped full-space curve is therefore not a universal consequence of finite memory; it requires additional mechanisms such as forced uniform allocation, interference, or an evaluation distribution concentrated on competing facts.

\paragraph{Structured facts.}
The preceding analysis treats all query labels as independent. 
Real knowledge often has structure: one base fact can support many derived queries. 
A simple deductive-closure model makes this explicit. 
Let $Z_1,\ldots,Z_{N_{\mathrm{base}}}$ be independent base labels, each uniformly distributed over $\mathcal{A}$. 
Each base label determines a block of $c$ derived queries through known one-to-one transformations, so $N=cN_{\mathrm{base}}$. 
If the learner observes $M$ base facts and $cM\le N$, the forced-compression analogue gives
\begin{equation}
\label{eq:structured-bound}
\begin{aligned}
\mathcal{E}_{\mathrm{struct}}
&\ge
\frac{cM}{N}
\delta^\star\!\left(\frac{B}{M}\right)\\
&\quad+
\left(1-\frac{cM}{N}\right)
\left(1-\frac{1}{K}\right).
\end{aligned}
\end{equation}
Compared with Eq.~\eqref{eq:rd-bound}, structure increases the number of surface queries covered per stored independent fact. 
This is the information-theoretic role of structure in our framework: reasoning does not remove the memory limit, but it can reduce the number of independent facts that must be compressed. 
The full derivation and the selective structured variant are in Appendix~\ref{app:deductive-proof}.

\section{Empirical Signatures of Coverage--Compression}
\label{sec:experiments}

The bound in Section~\ref{sec:theory} predicts a pattern rather than a single aggregate score. 
At fixed effective memory, increasing the number of independent facts should improve training coverage but reduce per-fact fidelity; increasing memory or adding structure should move the failure point. 
We test these predictions with theory-implied simulations and controlled real-model probes.

\subsection{Setup and diagnostics}
\label{sec:exp-setup}

We report three quantities throughout. 
\emph{Observed-query error} measures recall error on injected or observed facts and diagnoses compression distortion. 
\emph{Unobserved-query error} measures error on entities whose facts were not injected and diagnoses missing coverage. 
\emph{Full-space error} averages the two under a specified query mixture. 
We do not infer compression from aggregate accuracy alone; the key evidence is whether observed-query error changes with fact load under a controlled memory budget.

For the fact-injection probes, we generate synthetic independent facts
\[
e_i \mapsto a_i,\qquad a_i\in\{0,\ldots,K-1\},
\]
where entity strings are randomly generated and labels are sampled independently. 
This controlled setting follows recent work on new factual knowledge, knowledge editing, memorization, and knowledge capacity: coverage and fact load are known by construction, making observed-query error directly measurable \citep{gekhman2024newknowledge,huang2024hallueditbench,allenzhu2024knowledgecapacity,morris2025memorize,zhou2025taskstratifiedknowledge}. 
Unless otherwise stated, we use $K=10$, so chance-level error is $1-1/K=0.9$. 
Models are adapted on templates that state the entity-label relation and evaluated on held-out query forms: seen-direct, seen-paraphrase, seen-relation with a unique target answer, and unseen-entity.

We vary the number of injected facts $M$ and the LoRA rank used for adaptation \citep{hu2021lora,liu2024dora}. 
LoRA rank is used only as a controlled proxy for effective trainable memory; it should not be identified with the theoretical bit budget $B$. 
This proxy is controlled within a fixed target-module set, but target modules, optimization, model architecture, and quantization can still affect usable factual capacity. 
Hardware, training configurations, decoding rules, seeds, and numerical mean $\pm$ standard error tables are reported in Appendix~\ref{app:real-model-details}.

\begin{table*}[t]
\centering
\small
\caption{
Model suite used across the real-model probes. 
The suite is structured rather than exhaustive: Qwen and DeepSeek models are used for fact-injection probes, while Kimi-K2 and GLM-5 are used only for the long-context analogue.
}
\label{tab:model-suite}
\begin{tabular}{l l l}
\toprule
Model & Role & Main use \\
\midrule
Qwen3-8B & Dense baseline & Fact injection: full rank/load sweep \\
Qwen3-32B & Larger dense model & Fact injection: fixed-rank scale comparison \\
Qwen3-30B-A3B & MoE comparison & Fact injection: architecture comparison \\
DeepSeek-R1-Distill-Qwen-32B & Reasoning-distilled model & Fact injection: query-form robustness \\
Kimi-K2-Instruct & Frontier MoE reference & Long-context analogue only \\
GLM-5 family & Frontier reference & Long-context analogue only \\
\bottomrule
\end{tabular}
\end{table*}

The suite covers dense, MoE, and reasoning-distilled families. 
Qwen3 provides both dense and MoE variants, and DeepSeek-R1-Distill-Qwen-32B tests reasoning-style distillation in the fact-injection setting. 
Kimi-K2-Instruct and GLM-5 are used only as frontier references for the qualitative long-context analogue \citep{yang2025qwen3,deepseekai2025r1,kimi2025k2}. 
Hardware, training configurations, decoding rules, seeds, and representative numerical mean $\pm$ standard error tables are reported in Appendix~\ref{app:real-model-details}.

\subsection{Theory-implied envelope}
\label{sec:exp-theory-envelope}

We first instantiate the theoretical envelope itself. 
This component is not a language-model measurement. 
It numerically inverts the $K$-ary rate-distortion function and substitutes $\delta^\star(B/M)$ into the risk expressions from Section~\ref{sec:theory}. 
Its role is to make the predicted geometry explicit before testing whether the same signatures appear in model probes.

For unstructured forced compression,
\[
\mathcal{E}_{\mathrm{forced}}(M)
=
\frac{M}{N}\delta^\star\!\left(\frac{B}{M}\right)
+
\left(1-\frac{M}{N}\right)\left(1-\frac{1}{K}\right).
\]
For a structured deductive-closure setting in which one base fact supports $c$ derived queries,
\[
\mathcal{E}_{\mathrm{struct}}(M)
=
\frac{cM}{N}\delta^\star\!\left(\frac{B}{M}\right)
+
\left(1-\frac{cM}{N}\right)\left(1-\frac{1}{K}\right).
\]
For the long-context analogue,
\[
\mathcal{E}_{\mathrm{ctx}}(M_{\mathrm{ctx}})
=
\delta^\star\!\left(\frac{B_{\mathrm{ctx}}}{M_{\mathrm{ctx}}}\right).
\]

\begin{table*}[t]
\centering
\small
\caption{
Theory-implied operating points for $N=5{,}000$ and $K=10$. 
Values are computed from the exact rate-distortion envelope and are not model measurements.
}
\label{tab:theory-envelope}
\begin{tabular}{l c c c c}
\toprule
Setting & Parameters & Rate & Observed error & Full / conditional error \\
\midrule
Unstructured 
& $B=160,\ M=50$ 
& $3.20$ 
& $0.011$ 
& $0.891$ \\
Unstructured 
& $B=160,\ M=500$ 
& $0.32$ 
& $0.653$ 
& $0.875$ \\
Unstructured 
& $B=160,\ M=2000$ 
& $0.08$ 
& $0.787$ 
& $0.855$ \\
Memory scaling 
& $M=1000,\ B=640$ 
& $0.64$ 
& $0.531$ 
& $0.826$ \\
Memory scaling 
& $M=1000,\ B=2560$ 
& $2.56$ 
& $0.096$ 
& $0.739$ \\
Deductive closure 
& $B=160,\ M=250,\ c=10$ 
& $0.64$ 
& $0.531$ 
& $0.716$ \\
Long-context analogue 
& $B_{\mathrm{ctx}}=320,\ M_{\mathrm{ctx}}=100$ 
& $3.20$ 
& -- 
& $0.011$ \\
Long-context analogue 
& $B_{\mathrm{ctx}}=320,\ M_{\mathrm{ctx}}=1000$ 
& $0.32$ 
& -- 
& $0.653$ \\
\bottomrule
\end{tabular}
\end{table*}

The envelope separates effects that are otherwise conflated in aggregate accuracy. 
As $M$ grows at fixed $B$, missing coverage decreases, but the per-fact rate $B/M$ also decreases, so observed-fact error can rise even when full-space error improves. 
Increasing $B$ moves the system back along the distortion-rate curve, while structure increases query coverage per independent stored fact.

\subsection{Capacity-controlled fact injection}
\label{sec:exp-fact-injection}

We next ask whether the same signature appears in controlled model probes. 
The cleanest manipulation is to hold the adaptation mechanism fixed and vary the number of independent injected facts. 
If factual error is partly compression-induced, observed-query error should increase after the fact load exceeds the effective memory available to the adapter; increasing the memory proxy should shift the transition to larger $M$.

\begin{figure}[t]
\centering
\includegraphics[width=0.82\linewidth]{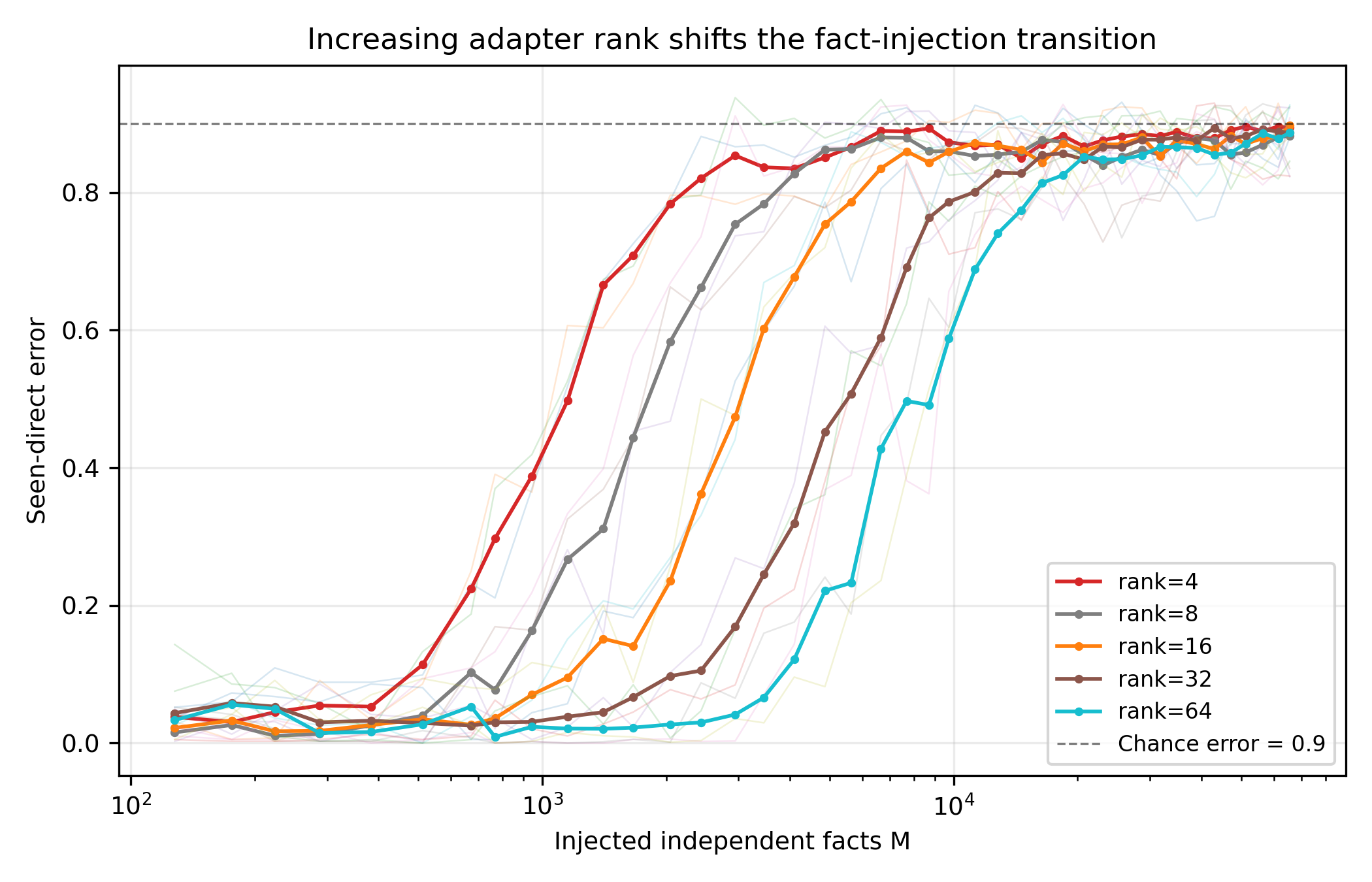}
\caption{
Increasing adapter capacity shifts the fact-injection transition. 
Qwen3-8B is adapted on independent synthetic facts while varying injected fact count $M$ and LoRA rank. 
Observed-query error rises as fact load increases at fixed rank, while higher rank delays the transition. 
The horizontal reference line denotes chance error for $K=10$. 
Representative numerical values are reported as mean $\pm$ standard error over three seeds in Appendix~\ref{app:real-model-details}.
}
\label{fig:qwen-rank-sweep}
\end{figure}

Figure~\ref{fig:qwen-rank-sweep} gives the most direct capacity test. 
At low fact load, all ranks store the injected mapping with low observed-query error. 
As $M$ increases, lower-rank adapters reach the high-error regime earlier, while larger ranks remain reliable over a wider range. 
This is the qualitative transition predicted by the coverage--compression bound.

\begin{figure}[t]
\centering
\includegraphics[width=0.82\linewidth]{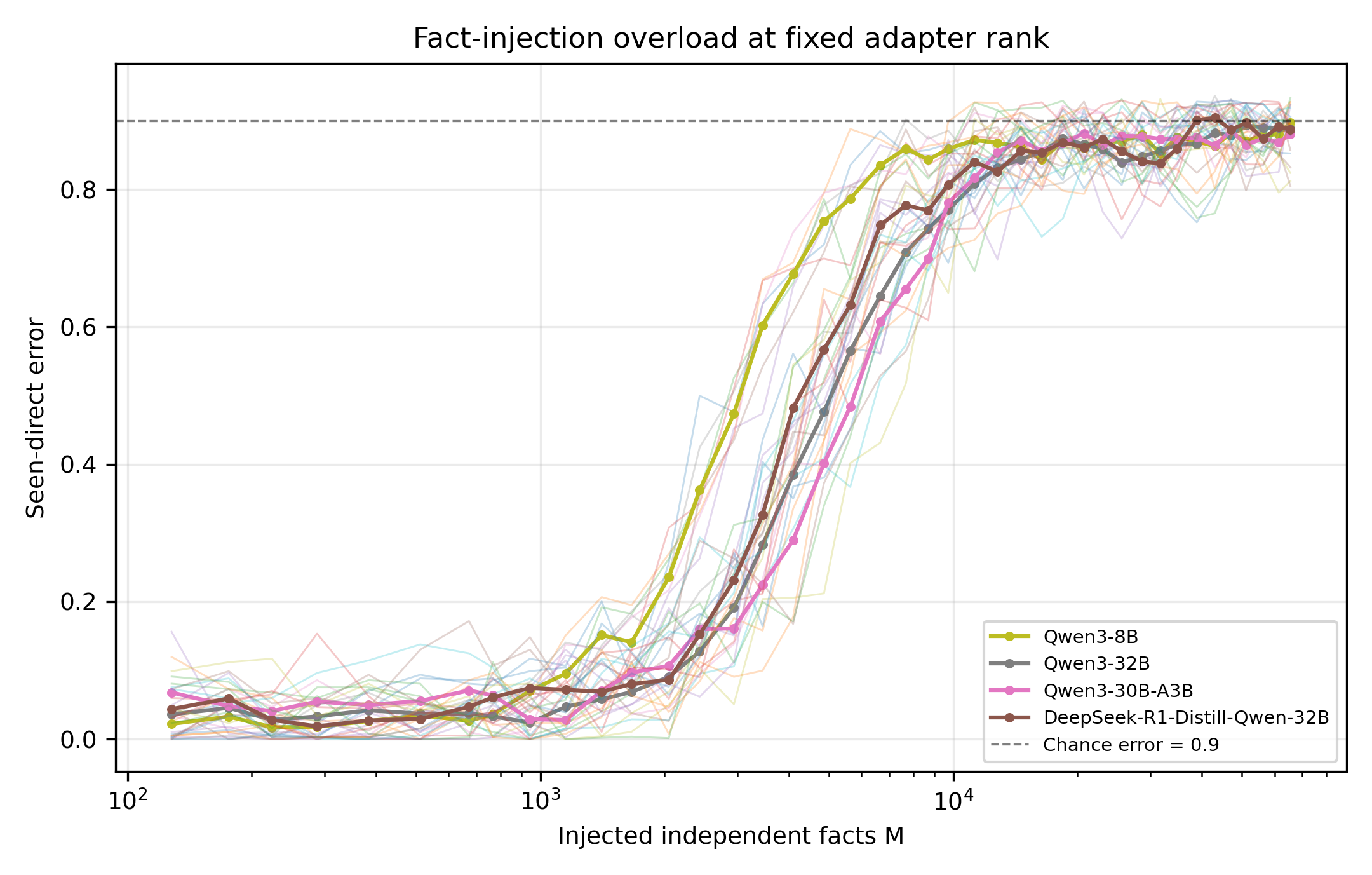}
\caption{
Compression overload at fixed adapter rank across model families. 
All models are evaluated at the same nominal LoRA rank on the seen-direct split. 
The error curves rise with the number of injected independent facts, approaching chance level for $K=10$ at high load. 
Representative numerical values are reported as mean $\pm$ standard error over three seeds in Appendix~\ref{app:real-model-details}.
}
\label{fig:rank16-model-sweep}
\end{figure}

The fixed-rank comparison in Figure~\ref{fig:rank16-model-sweep} suggests that the pattern is not confined to one checkpoint. 
The transition point varies across models, but the direction is consistent: when independent injected facts grow under a fixed adapter budget, observed-fact fidelity degrades. 
This supports the theory at the level of diagnostic signatures rather than exact numerical agreement with the idealized rate-distortion curve.

\subsection{Robustness and structure}
\label{sec:exp-robust-structure}

A possible failure mode of synthetic fact injection is template memorization. 
We therefore evaluate multiple query forms for the same injected facts. 
The seen-relation split is restricted to query forms with a unique target answer; ambiguous inverse queries are excluded.

\begin{figure}[t]
\centering
\includegraphics[width=0.82\linewidth]{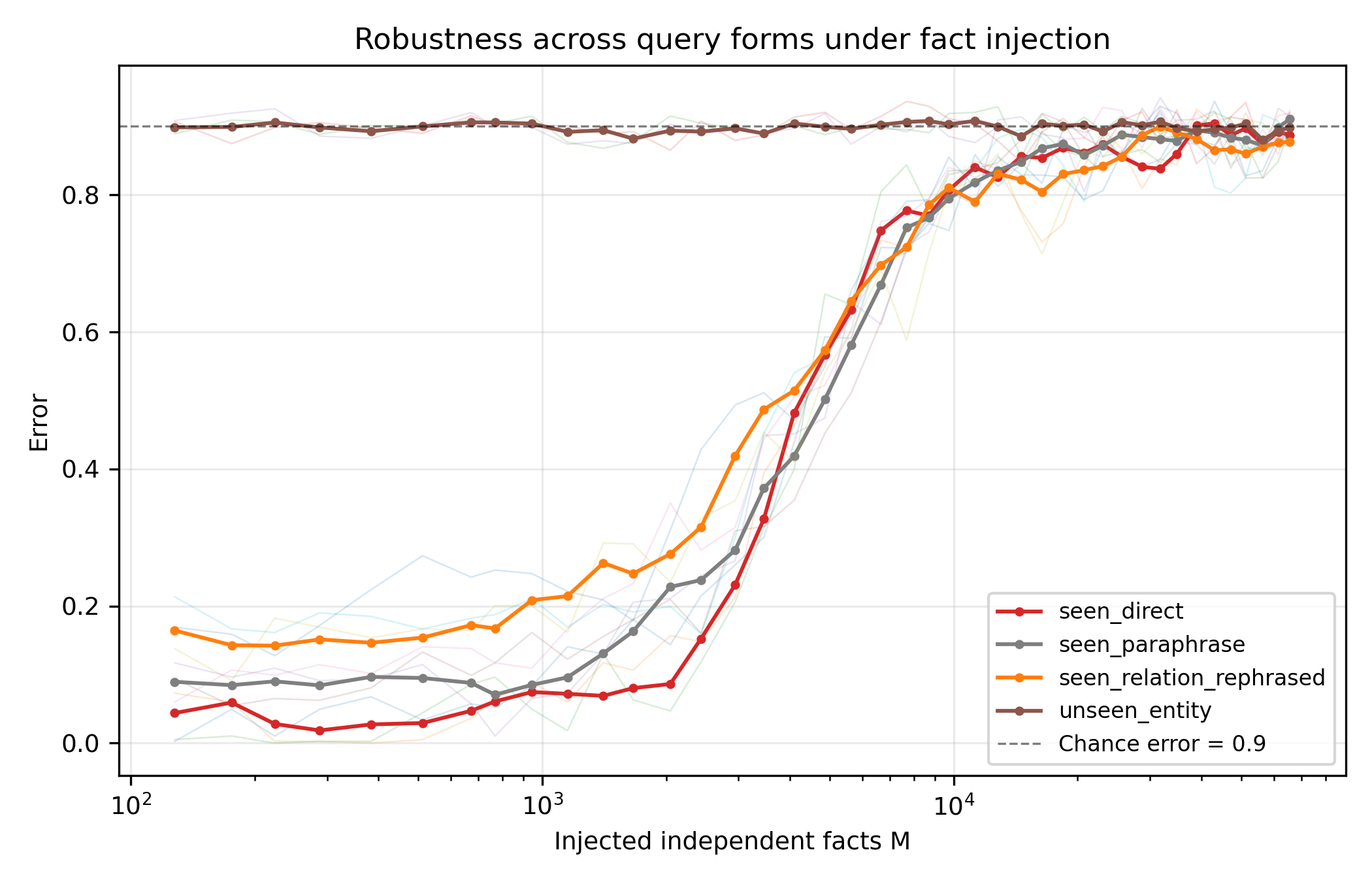}
\caption{
Query-form robustness under fact injection. 
DeepSeek-R1-Distill-Qwen-32B is evaluated on direct, paraphrased, relation-rephrased, and unseen-entity queries. 
Unseen entities remain near the chance error of $0.9$ for $K=10$, while paraphrased and relation-rephrased queries expose additional recall fragility as fact load grows. 
Representative numerical values are reported as mean $\pm$ standard error over three seeds in Appendix~\ref{app:real-model-details}.
}
\label{fig:split-robustness}
\end{figure}

Figure~\ref{fig:split-robustness} separates coverage from robustness. 
Unseen entities remain close to chance, as expected for facts that were never injected. 
The seen splits are below chance at lower loads but degrade as $M$ increases, with paraphrased and relation-rephrased queries degrading earlier than direct queries. 
Exposure alone is not sufficient; the fact must be represented with enough fidelity to survive changes in query form.

The theory also predicts that surface query count is not the right memory-load measure when many queries share the same base fact. 
We compare an unstructured condition, where each surface query has an independent label, with a structured condition, where derived queries are determined by fewer independent base labels.

\begin{figure}[t]
\centering
\includegraphics[width=0.82\linewidth]{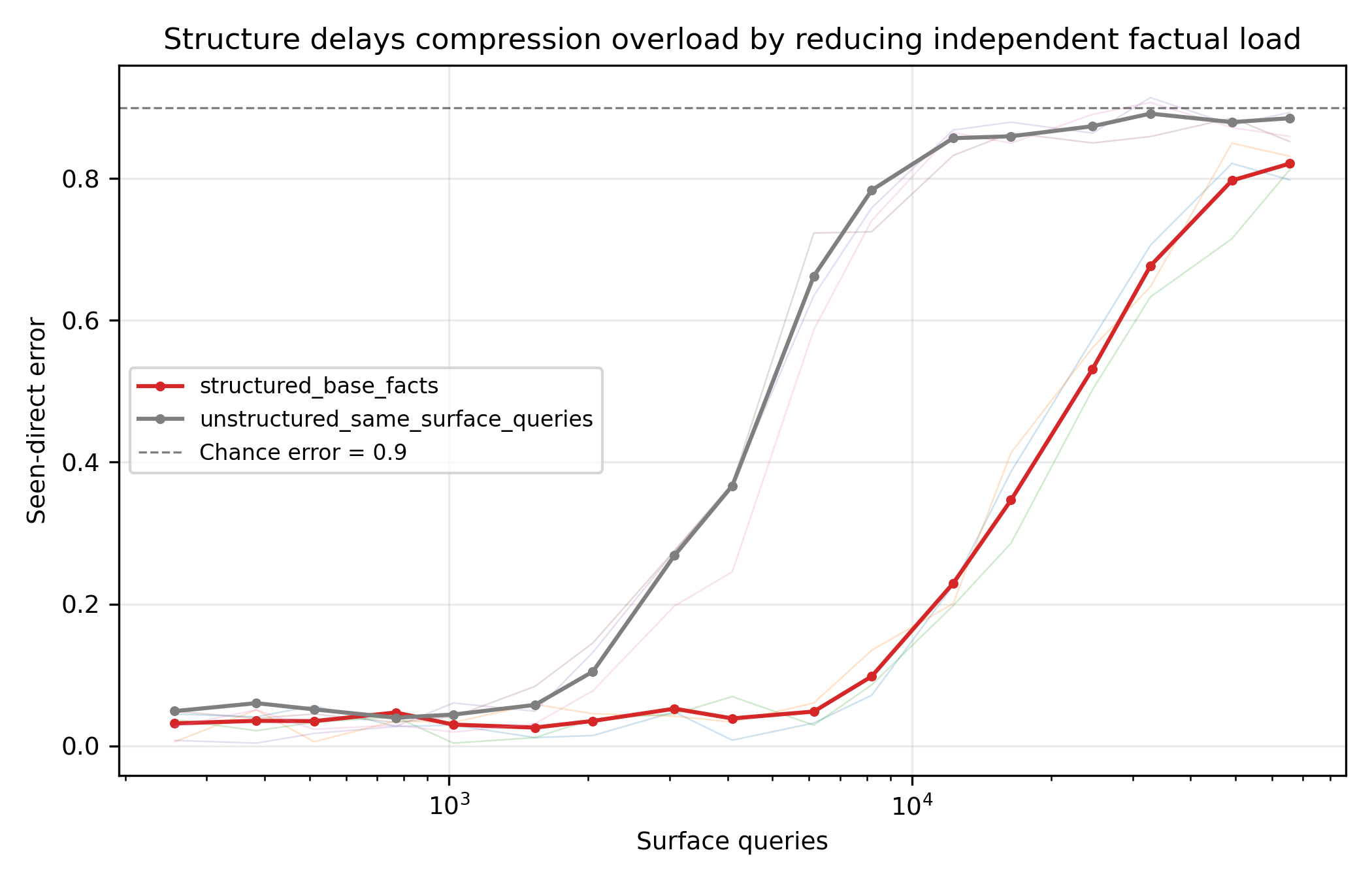}
\caption{
Structure delays compression overload by reducing independent factual load. 
The unstructured condition assigns independent labels to surface queries. 
The structured condition uses fewer independent base labels, each supporting multiple derived queries. 
Representative numerical values are reported as mean $\pm$ standard error over three seeds in Appendix~\ref{app:real-model-details}.
}
\label{fig:structure}
\end{figure}

Figure~\ref{fig:structure} supports the deductive-closure prediction. 
When each surface query corresponds to an independent label, error rises quickly with the number of queries. 
When queries share base facts, the same surface-query count imposes a smaller independent memory load, and the overload point is delayed. 
This does not claim that all reasoning reduces to deductive closure; it shows the information-theoretic role of structure in the present framework.

\subsection{Long-context analogue and summary}
\label{sec:exp-long-context}

Finally, we apply the same diagnostic lens to an in-context setting. 
Here the facts are present in the prompt, so the issue is not parametric coverage. 
The question is whether many independent key-value facts compete for limited effective working memory or attention, especially when retrieval cannot rely on literal overlap. 
This setup follows recent long-context evaluations showing that nominal context length does not necessarily imply reliable recall under multi-fact, long in-context learning, latent-structure, or non-literal retrieval demands \citep{hsieh2024ruler,zhang2024infinitebench,bai2024longbenchv2,li2024needlebench,vodrahalli2024michelangelo,li2024longiclbench,modarressi2025nolima}.

\begin{figure}[t]
\centering
\includegraphics[width=0.82\linewidth]{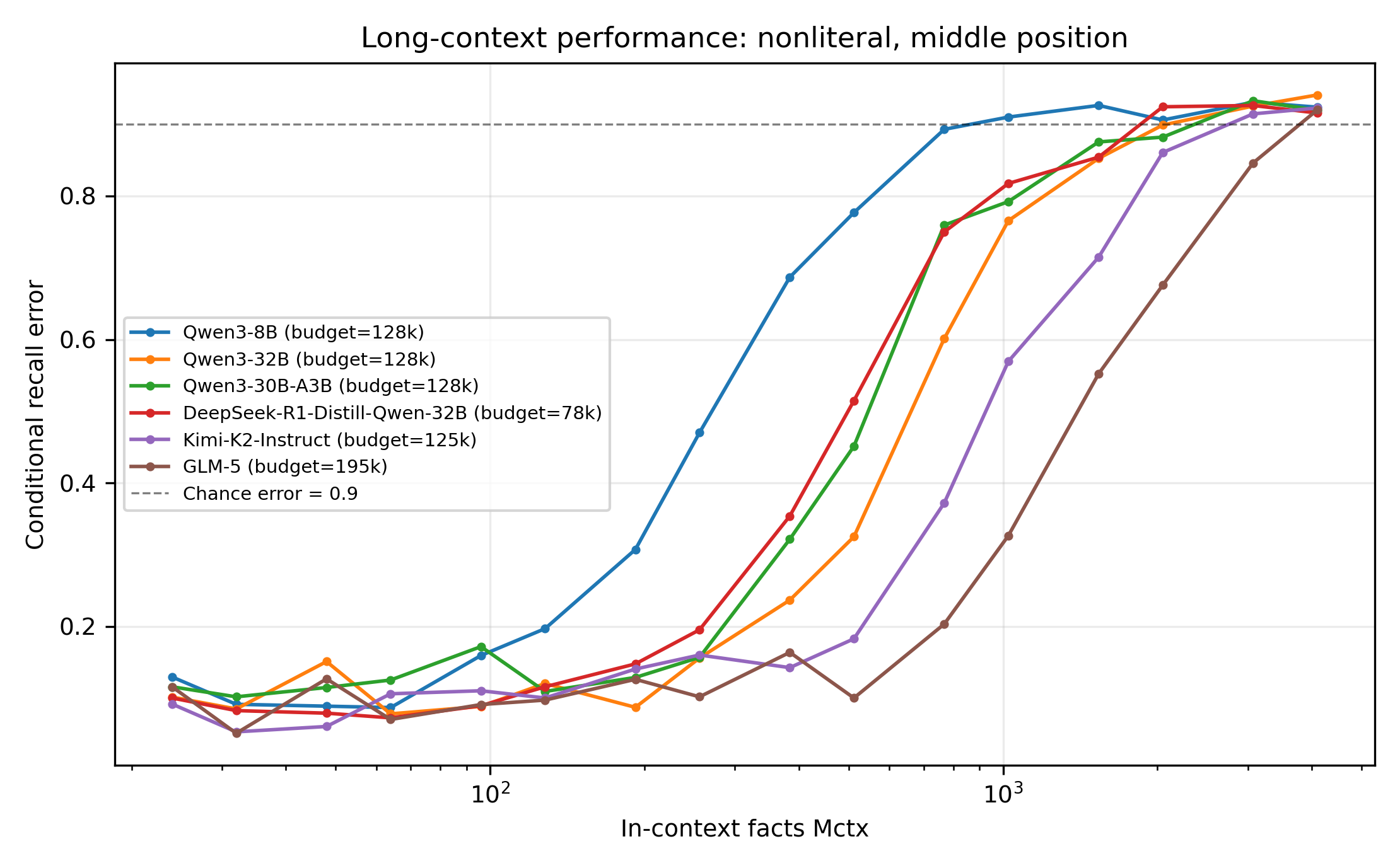}
\caption{
Long-context analogue under non-literal middle-position retrieval. 
Models are given many in-context key-value facts and queried with non-literal prompts targeting facts placed near the middle of the context. 
Error increases as the number of in-context facts grows, suggesting an analogous competition effect. 
This figure is intended as a qualitative analogue, not as a direct estimate of $B_{\mathrm{ctx}}$ or a quantitative test of Theorem~\ref{thm:rd}.
}
\label{fig:long-context}
\end{figure}

The long-context result is an analogue rather than a direct test of the parametric memory theorem. 
It shows that the same diagnostic distinction remains useful when facts are present but must be selected and used from a crowded context. 
Across probes, the same pattern recurs: increasing independent fact load under fixed effective memory raises observed-query error; increasing adapter rank shifts the transition; unseen entities remain near chance; and structure delays overload by reducing the number of independent labels that must be stored. 
The results should not be read as evidence that language models implement the ideal rate-distortion code. 
They support the narrower claim that coverage and compression are empirically separable, and that observed-fact error provides a direct diagnostic for lossy factual recall under finite effective memory.

\section{Discussion and Conclusion}
\label{sec:discussion-conclusion}

The coverage--compression view changes how closed-book factual error should be interpreted. 
A factual answer can be wrong either because the relevant fact was not effectively covered, or because the fact was observed but stored with insufficient fidelity. 
The main bound separates these mechanisms: unseen-entity error diagnoses missing coverage, while observed-query error diagnoses lossy recall of injected facts.

The bound should be read as a limit on unstructured factual memorization, not as a complete theory of hallucination. 
In the random-source model, answers are independent, so correct prediction must come from stored information rather than semantic regularity. 
When $B/M<\log_2K$, the effective memory per observed fact is insufficient for lossless recall, giving a nonzero observed-fact error floor. 
This identifies a capacity scale, not an optimal dataset size: more observations need not hurt an optimal learner, while forced shared representations or interference can degrade observed-fact fidelity.

The same decomposition explains why common interventions help. 
Retrieval changes the resource model by moving facts out of parametric memory and shifting the bottleneck to retrieval and verification \citep{yan2024crag,ru2024ragchecker}. 
Abstention changes the operating point by avoiding queries outside reliable coverage. 
Structure and reasoning reduce the number of independent facts that must be stored, while long-context organization addresses an analogous working-memory problem in which present facts still compete for limited attention or effective context capacity \citep{hsieh2024ruler,modarressi2025nolima}.

The controlled probes support this diagnostic interpretation rather than an exact numerical identification with the idealized rate-distortion curve. 
At fixed adapter rank, increasing the number of injected independent facts raises observed-query error; increasing rank shifts the transition; unseen entities remain near chance; and structured facts delay overload.

\section*{Limitations}

The model is intentionally minimal. 
The uniform mapping is a maximally unstructured random-source model; real knowledge contains semantic, logical, temporal, and distributional redundancy. 
The bit budget $B$ is an abstract effective-memory measure, not a direct parameter-count law; in the model probes, LoRA rank is only a controlled capacity proxy \citep{hu2021lora}. 
The derivation also uses a free-addressing convention and asymptotic rate-distortion theory, so explicit address costs and finite-blocklength effects may matter.

Finally, the empirical study uses controlled synthetic fact injection rather than naturally occurring factual knowledge. 
This design gives known coverage, known independent fact load, and measurable observed-query error, but it does not cover every source of open-domain hallucination, such as conflicting evidence, outdated knowledge, decoding behavior, calibration failure, or instruction-following pressure \citep{wei2024simpleqa,wei2024longformfactuality}.

\bibliography{custom}

\newpage
\appendix
\newpage
\appendix

\section{Technical Details and Proofs}
\label{app:proofs}

This appendix gives the technical details omitted from the main text. 
The main paper states the key proof ideas; here we record the full derivations, the free-addressing convention, the tightness statement, and the allocation variants used for interpretation.

\subsection{Memory Accounting and Rate-Distortion Preliminaries}
\label{app:rd-prelim}

\paragraph{Free addressing.}
Throughout the main theorem, we use the optimistic free-addressing convention. 
The training index set $S$ and the test query identity $q$ are available to the decoder as addresses. 
The budget $B$ is charged only to stored answer information. 
If the learner must also encode which queries are stored, an additional address cost on the order of
\[
\log_2 \binom{N}{M}
\]
is required. 
Thus the free-addressing convention favors the learner; any explicit address cost only tightens the memory constraint.

\paragraph{The $K$-ary distortion-rate function.}
Let $U$ be uniformly distributed over a $K$-symbol alphabet. 
Under zero-one distortion,
\[
d(u,\hat u)=\mathbf{1}\{u\neq \hat u\},
\]
the rate-distortion function of the uniform $K$-ary source is
\begin{equation}
\label{eq:app-rd}
\begin{aligned}
R_K(D)
&=
\log_2 K
-
h_2(D)
-
D\log_2(K-1),\\
&\hspace{2.0em}
0\le D\le 1-\frac{1}{K},
\end{aligned}
\end{equation}
where
\[
h_2(D)
=
-D\log_2D-(1-D)\log_2(1-D).
\]
For $D\ge 1-1/K$, the rate is zero, since random guessing already achieves distortion $1-1/K$. 
We define the inverse distortion-rate function as
\begin{equation}
\label{eq:app-delta}
\delta^\star(r)
=
\inf\{D\ge0:R_K(D)\le r\}.
\end{equation}
In particular,
\[
\delta^\star(r)=0
\quad\text{for}\quad
r\ge \log_2 K,
\]
which is the lossless memorization regime.

\paragraph{Counting intuition.}
A block of $M$ independent $K$-ary labels has $K^M$ possible sequences, requiring $M\log_2K$ bits for lossless storage. 
If a fraction $D$ of positions may be reconstructed incorrectly, then one reconstructed sequence represents roughly
\[
2^{M[h_2(D)+D\log_2(K-1)]}
\]
nearby true sequences: choose about $DM$ error positions and one of $K-1$ alternative labels at each. 
Dividing the total number of sequences by this ball size and taking logarithms gives the per-symbol rate in Eq.~\eqref{eq:app-rd}. 
Thus, allowing distortion reduces the number of bits needed per fact.

\subsection{Coverage, Compression, and Tightness}
\label{app:rd-proof}

\paragraph{Coverage-only floor.}
\label{app:coverage-proof}
Let $S\subset\mathcal{Q}$ be the training set. 
For any query $q\notin S$, the label $g(q)$ is uniformly distributed over $\mathcal{A}$ and independent of the observed labels. 
Therefore, for any deterministic or randomized predictor,
\[
\Pr[\hat g(q)=g(q)\mid q\notin S]\le \frac{1}{K},
\]
and hence
\[
\Pr[\hat g(q)\neq g(q)\mid q\notin S]\ge 1-\frac{1}{K}.
\]
A uniformly drawn test query lies outside $S$ with probability $(N-M)/N$. 
Even if the learner answers every query in $S$ perfectly, the expected full-space error is at least
\begin{equation}
\label{eq:app-coverage}
\mathcal{E}
\ge
\left(1-\frac{M}{N}\right)
\left(1-\frac{1}{K}\right).
\end{equation}

\begin{theorem}[Coverage--compression bound]
\label{thm:rd-app}
Under the setting of Section~\ref{sec:setting}, any learner with memory budget $B$ satisfies
\begin{equation}
\label{eq:app-rd-bound}
\mathcal{E}
\ge
\frac{M}{N}
\delta^\star\!\left(\frac{B}{M}\right)
+
\left(1-\frac{M}{N}\right)
\left(1-\frac{1}{K}\right).
\end{equation}
\end{theorem}

\begin{proof}
Condition on the sampled training set $S$. 
The observed labels
\[
X_S=(g(q))_{q\in S}
\]
are $M$ independent uniform draws from $\mathcal{A}$. 
The learner encodes the observed information into a memory state $W$ with at most $2^B$ possible values. 
Therefore,
\[
H(W\mid S)\le B.
\]
Let
\[
\widehat X_S=(\hat g(q))_{q\in S}
\]
be the learner's reconstructions on the training queries. 
Under free addressing, $\widehat X_S$ is produced from $(S,W)$ together with decoder randomness. 
Thus, by the data-processing inequality,
\begin{equation}
\label{eq:app-dpi}
\begin{aligned}
I(X_S;\widehat X_S\mid S)
&\le I(X_S;W\mid S)\\
&\le H(W\mid S)\\
&\le B.
\end{aligned}
\end{equation}

Let
\[
D_S
=
\frac{1}{M}
\sum_{q\in S}
\Pr[\hat g(q)\neq g(q)\mid S]
\]
be the average reconstruction error on the observed facts. 
By the converse of rate-distortion theory for a memoryless uniform $K$-ary source,
\[
\frac{1}{M}I(X_S;\widehat X_S\mid S)
\ge
R_K(D_S).
\]
Combining this with Eq.~\eqref{eq:app-dpi} gives
\[
R_K(D_S)\le \frac{B}{M}.
\]
By the definition of $\delta^\star$, this implies
\[
D_S
\ge
\delta^\star\!\left(\frac{B}{M}\right).
\]

A uniformly drawn test query lies in $S$ with probability $M/N$, contributing observed-fact error at least $\delta^\star(B/M)$. 
It lies outside $S$ with probability $1-M/N$, in which case its label is independent of $(S,W)$ and the best possible error is at least $1-1/K$. 
Combining the two cases yields Eq.~\eqref{eq:app-rd-bound}.
\end{proof}

\paragraph{Tightness under free addressing.}
\label{app:tightness}
The lower bound is tight up to finite-blocklength effects under the same free-addressing convention. 
Condition on $S$ and encode the sequence $X_S$ using a block rate-distortion code for the uniform $K$-ary source at rate $B/M$. 
For any $D>\delta^\star(B/M)$, classical rate-distortion achievability gives codes whose observed-fact distortion approaches $D$ as $M$ grows. 
The decoder answers observed queries from the reconstructed sequence and guesses uniformly on unobserved queries. 
Therefore,
\begin{equation}
\label{eq:app-achievability}
\begin{aligned}
\mathcal{E}
&\le
\frac{M}{N}
\left[
\delta^\star\!\left(\frac{B}{M}\right)+o(1)
\right] \\
&\quad+
\left(1-\frac{M}{N}\right)
\left(1-\frac{1}{K}\right).
\end{aligned}
\end{equation}
Thus Eq.~\eqref{eq:app-rd-bound} is the natural rate-distortion envelope for the random-source model, not merely a loose lower bound.

\subsection{Allocation Regimes}
\label{app:selective-memory}

The main theorem describes the forced-compression regime in which all $M$ observed facts share one $B$-bit representation. 
For interpretation, it is useful to compare this with an idealized selective-memory policy.

\paragraph{Selective memory.}
A selective learner chooses an effective number $L\le M$ of facts to store, spends the memory budget on those facts, and guesses on the rest. 
The case $L=0$ corresponds to guessing everywhere, so the allocation curve is
\begin{equation*}
\label{eq:app-selective}
\begin{aligned}
\mathcal{E}_{\mathrm{sel}}(M)
&=
\min\left\{
1-\frac{1}{K},
\min_{1\le L\le M}\Phi(L)
\right\},\\
\Phi(L)
&=
\frac{L}{N}
\delta^\star\!\left(\frac{B}{L}\right)
+
\left(1-\frac{L}{N}\right)
\left(1-\frac{1}{K}\right)
\end{aligned}
\end{equation*}
Because the feasible set grows with $M$, $\mathcal{E}_{\mathrm{sel}}(M)$ is non-increasing in $M$. 
This curve captures memory-limited diminishing returns; it should not be read as saying that more data hurts an optimal learner.

The lossless endpoint gives the scale
\[
L_{\mathrm{lossless}}
\approx
\frac{B}{\log_2K}.
\]
This is the approximate number of unstructured $K$-ary facts that can be stored without error under free addressing.

\paragraph{Forced compression.}
If all $M$ observed facts must share the same $B$-bit representation, the operating curve is
\begin{equation*}
\label{eq:app-forced}
\mathcal{E}_{\mathrm{forced}}(M)
=
\frac{M}{N}\delta^\star\!\left(\frac{B}{M}\right)
+
\left(1-\frac{M}{N}\right)\left(1-\frac{1}{K}\right).
\end{equation*}
As $M$ grows, the unseen-query term decreases because more queries are covered, but the observed-fact distortion can increase because the per-fact rate $B/M$ decreases. 
This is the compression-overload effect. 
A U-shaped full-space curve is not implied by finite memory alone; it requires an additional mechanism such as forced uniform allocation, interference, or an evaluation distribution concentrated on competing facts.

\subsection{Structure and Abstention}
\label{app:deductive-proof}

\paragraph{Deductive closure.}
The structured setting formalizes the idea that deterministic structure reduces the number of independent facts that must be stored. 
Let there be $N_{\mathrm{base}}$ independent base facts
\[
Z_1,\ldots,Z_{N_{\mathrm{base}}},
\qquad
Z_i\sim\mathrm{Unif}(\mathcal{A}).
\]
Each base fact determines a block of $c$ derived queries, so the total number of queries is
\[
N=cN_{\mathrm{base}}.
\]
For each derived query in the block associated with $Z_i$, assume the answer is a known one-to-one transformation of $Z_i$. 
Thus, an error in reconstructing $Z_i$ induces the same zero-one error on its derived queries.

\begin{proposition}[Deductive closure reduces effective memory load]
\label{prop:deductive}
Suppose the learner observes $M$ base facts and compresses their labels into $B$ bits. 
Under the forced-compression analogue of Theorem~\ref{thm:rd-app}, and assuming $cM\le N$, the full-space error satisfies
\begin{equation*}
\label{eq:app-deductive}
\mathcal{E}_{\mathrm{struct}}
\ge
\frac{cM}{N}
\delta^\star\!\left(\frac{B}{M}\right)
+
\left(1-\frac{cM}{N}\right)
\left(1-\frac{1}{K}\right).
\end{equation*}
\end{proposition}

\begin{proof}
The $M$ observed base labels are independent uniform $K$-ary symbols. 
By the same rate-distortion argument as in Theorem~\ref{thm:rd-app}, their average reconstruction error is at least $\delta^\star(B/M)$. 
Each reconstructed base fact supports $c$ derived queries, and the one-to-one maps propagate base-label errors to derived-query errors. 
Thus, the observed base facts cover $cM$ of the $N$ derived queries. 
The remaining queries depend on unobserved base facts and have guessing error at least $1-1/K$. 
Combining the covered and uncovered parts gives Eq.~\eqref{eq:app-deductive}.
\end{proof}

A selective structured variant stores only $L\le M$ base facts with nontrivial fidelity:
\begin{equation*}
\label{eq:app-deductive-selective}
\begin{aligned}
\mathcal{E}_{\mathrm{struct,sel}}
&=
\min\left\{
1-\frac{1}{K},
\min_{\substack{1\le L\le M\\ cL\le N}}
\Psi(L)
\right\},\\
\Psi(L)
&=
\frac{cL}{N}
\delta^\star\!\left(\frac{B}{L}\right)
+
\left(1-\frac{cL}{N}\right)
\left(1-\frac{1}{K}\right).
\end{aligned}
\end{equation*}
Compared with the unstructured case, the factor $c$ increases query coverage per stored independent fact. 
This is the formal sense in which deductive closure relaxes the hallucination floor.

\paragraph{Abstention.}
\label{app:abstention}
The main setting requires the learner to answer every query. 
If abstention is allowed, the learner can restrict itself to a subset of queries for which it has reliable memory. 
Suppose the learner answers only a fraction $1-\gamma$ of the query space and abstains on the remaining fraction $\gamma$. 
Let $L$ be the number of effectively stored facts among the answered queries. 
Under free addressing, the conditional error on answered queries is
\begin{equation}
\label{eq:app-abstention}
\begin{aligned}
\mathcal{E}_{\mathrm{ans}}
&=
\frac{L}{(1-\gamma)N}
\delta^\star\!\left(\frac{B}{L}\right) \\
&\quad+
\left(
1-\frac{L}{(1-\gamma)N}
\right)
\left(1-\frac{1}{K}\right).
\end{aligned}
\end{equation}
subject to
\[
0\le L\le M,
\qquad
0\le \gamma<1,
\qquad
L\le (1-\gamma)N.
\]
This expression makes explicit the coverage--accuracy trade-off: abstention allows memory to be concentrated on fewer answered queries, but lowers answer coverage.

\section{Theory-Implied Simulation Details}
\label{app:simulation-details}

The main text reports representative operating points from the rate-distortion envelope. 
Here we give the full theory-implied sweeps used to construct Table~\ref{tab:theory-envelope}. 
All values are computed by numerically inverting the $K$-ary distortion-rate function with $K=10$ and rounding to three decimals. 
These are not language-model measurements; they instantiate the idealized envelope analyzed in Section~\ref{sec:theory}.

\paragraph{Numerical inversion.}
For a target rate $r$, we compute
\[
\delta^\star(r)=\inf\{D\ge0:R_K(D)\le r\}
\]
by one-dimensional bisection over $D\in[0,1-1/K]$. 
If $r\ge \log_2K$, we set $\delta^\star(r)=0$. 
If $r\le 0$, we set $\delta^\star(r)=1-1/K$. 
The reported full-space risks are then obtained by substituting $\delta^\star(B/M)$ into the corresponding operating curves.

\subsection{Unstructured Coverage--Compression Decomposition}
\label{app:sim-unstructured}

We use $N=5{,}000$, $K=10$, and $B=160$ bits. 
The coverage-only floor is
\begin{equation*}
\mathcal{E}_{\mathrm{cov}}(M)
=
\left(1-\frac{M}{N}\right)
\left(1-\frac{1}{K}\right).
\end{equation*}

For the lossless selective baseline, we use
\begin{equation*}
\begin{aligned}
L_{\mathrm{lossless}}
&=
\min\left(M,\left\lfloor\frac{B}{\log_2K}\right\rfloor\right),\\
\mathcal{E}_{\mathrm{sel,lossless}}(M)
&=
\left(1-\frac{L_{\mathrm{lossless}}}{N}\right)
\left(1-\frac{1}{K}\right).
\end{aligned}
\end{equation*}

The forced-compression curve is
\begin{equation*}
\begin{aligned}
\mathcal{E}_{\mathrm{forced}}(M)
&=
\frac{M}{N}\delta^\star\!\left(\frac{B}{M}\right) \\
&\quad+
\left(1-\frac{M}{N}\right)
\left(1-\frac{1}{K}\right).
\end{aligned}
\end{equation*}

\begin{table*}[htb]
\centering
\small
\caption{
Theory-implied sweep for unstructured factual recall with $N=5{,}000$, $K=10$, and $B=160$ bits.
}
\label{tab:app-unstructured}
\begin{tabular}{c c c c c c}
\toprule
$M$ 
& $B/M$ 
& $\delta^\star(B/M)$
& Coverage floor
& Selective lossless
& Forced risk \\
\midrule
25   & 6.400 & 0.000 & 0.896 & 0.896 & 0.896 \\
50   & 3.200 & 0.011 & 0.891 & 0.891 & 0.891 \\
100  & 1.600 & 0.275 & 0.882 & 0.891 & 0.888 \\
250  & 0.640 & 0.531 & 0.855 & 0.891 & 0.882 \\
500  & 0.320 & 0.653 & 0.810 & 0.891 & 0.875 \\
1000 & 0.160 & 0.734 & 0.720 & 0.891 & 0.867 \\
2000 & 0.080 & 0.787 & 0.540 & 0.891 & 0.855 \\
5000 & 0.032 & 0.831 & 0.000 & 0.891 & 0.831 \\
\bottomrule
\end{tabular}
\end{table*}

\subsection{Memory Scaling}
\label{app:sim-memory}

\subsection{Deductive Closure}
\label{app:sim-structure}

We use $N=5{,}000$, $K=10$, $B=160$, and closure factor $c=10$. 
Each observed base fact supports $c$ derived queries. 
The structured operating curve is
\[
\mathcal{E}_{\mathrm{struct}}(M)
=
\frac{cM}{N}\delta^\star\!\left(\frac{B}{M}\right)
+
\left(1-\frac{cM}{N}\right)
\left(1-\frac{1}{K}\right)
\]
for $cM\le N$.

\begin{table*}[htb]
\centering
\small
\caption{
Theory-implied sweep for deductive closure with $N=5{,}000$, $K=10$, $B=160$, and $c=10$.
}
\label{tab:app-deductive}
\begin{tabular}{c c c c}
\toprule
Observed base facts $M$
& Effective covered queries $cM$
& Unstructured risk
& Structured risk \\
\midrule
25  & 250  & 0.896 & 0.855 \\
50  & 500  & 0.891 & 0.811 \\
100 & 1000 & 0.888 & 0.775 \\
250 & 2500 & 0.882 & 0.716 \\
500 & 5000 & 0.875 & 0.653 \\
\bottomrule
\end{tabular}
\end{table*}

\subsection{Long-Context Analogue}
\label{app:sim-long-context}

For the long-context analogue, we set $K=10$ and $B_{\mathrm{ctx}}=320$ bits. 
The reported value is the predicted conditional recall error on in-context facts:
\[
\mathcal{E}_{\mathrm{ctx}}(M_{\mathrm{ctx}})
=
\delta^\star\!\left(\frac{B_{\mathrm{ctx}}}{M_{\mathrm{ctx}}}\right).
\]
This is a theory-implied analogue, not a measurement of any particular long-context model.

\begin{table}[!t]
\centering
\small
\caption{
Theory-implied sweep for the long-context factual recall analogue with $K=10$ and $B_{\mathrm{ctx}}=320$ bits.
}
\label{tab:app-long-context}
\begin{tabular}{c c c}
\toprule
$M_{\mathrm{ctx}}$
& $B_{\mathrm{ctx}}/M_{\mathrm{ctx}}$
& Predicted recall error \\
\midrule
50   & 6.400 & 0.000 \\
100  & 3.200 & 0.011 \\
250  & 1.280 & 0.350 \\
500  & 0.640 & 0.531 \\
1000 & 0.320 & 0.653 \\
2000 & 0.160 & 0.734 \\
\bottomrule
\end{tabular}
\end{table}

\section{Real-Model Probe Details}
\label{app:real-model-details}

This section reports the implementation details for the controlled model probes in Section~\ref{sec:experiments}. 
The goal is to make the empirical signatures auditable: the injected facts are known by construction, the unseen-entity split is disjoint from the injected facts, and observed-query error can be measured directly. 
Unless otherwise stated, reported errors are mean $\pm$ standard error over seeds $\{3407,2025,1337\}$.

\subsection{Hardware, Software, and Data}
\label{app:hardware-data}

All adaptation runs were performed on a single node with eight NVIDIA A100-SXM4-80GB GPUs connected by NVLink. 
We used distributed data parallel training with \texttt{torchrun}, bf16 precision, CUDA 12.1, NVIDIA driver 535.xx, PyTorch 2.4.x, Transformers 4.45.x, PEFT 0.12.x, and Accelerate 0.34.x.

\begin{table*}[t]
\centering
\caption{
Hardware and software environment for real-model probes.
}
\label{tab:hardware-software}
\begin{tabular*}{\textwidth}{@{\extracolsep{\fill}} l l}
\toprule
Field & Value \\
\midrule
Node count & 1 \\
GPU count & 8 \\
GPU type & NVIDIA A100-SXM4-80GB \\
Interconnect & NVLink \\
Precision & bf16 \\
Distributed backend & \texttt{torchrun\_ddp} \\
CUDA & 12.1 \\
NVIDIA driver & 535.xx \\
PyTorch & 2.4.x \\
Transformers & 4.51.x \\
PEFT & 0.12.x \\
Accelerate & 0.34.x \\
\bottomrule
\end{tabular*}
\end{table*}

Each synthetic fact is an entity-label mapping
\[
e_i \mapsto a_i,
\qquad
a_i\in\{0,\ldots,K-1\}.
\]
We use $K=10$, so the chance-level error is $1-1/K=0.9$. 
Labels are sampled uniformly from $\{0,\ldots,9\}$. 
Each fact is expressed with two training templates. 
The unseen-entity pool is the same size as the injected pool for each fact load, and train/test entity overlap is disallowed.

\begin{table*}[t]
\centering
\caption{
Synthetic fact-injection data construction.
}
\label{tab:fact-data}
\begin{tabular*}{\textwidth}{@{\extracolsep{\fill}} l l}
\toprule
Field & Value \\
\midrule
Answer alphabet size & $K=10$ \\
Chance error & 0.900 \\
Fact loads & $M\in\{512,2048,8192,28672\}$ \\
Label distribution & Uniform over $\{0,\ldots,9\}$ \\
Training templates per fact & 2 \\
Evaluation splits & Seen-direct, seen-paraphrase, seen-relation, unseen-entity \\
Unseen pool size & Same as injected pool \\
Train/test entity overlap & False \\
Answer normalization & Strip whitespace; parse first integer in $\{0,\ldots,9\}$; invalid output counted wrong \\
Seeds & 3407, 2025, 1337 \\
\bottomrule
\end{tabular*}
\end{table*}

\subsection{Adaptation and Decoding Configuration}
\label{app:adaptation-config}

All fact-injection probes use LoRA adapters. 
LoRA rank is treated as a controlled capacity proxy and is not interpreted as the theoretical bit budget $B$. 
Across models, LoRA is applied to \texttt{q\_proj}, \texttt{v\_proj}, and \texttt{o\_proj}; LoRA dropout is 0.05 and LoRA alpha is set to $2r$, where $r$ is the LoRA rank. 
We use gradient checkpointing, cosine learning-rate scheduling, warmup ratio 0.03, weight decay 0.0, maximum gradient norm 1.0, and maximum sequence length 96. 
Checkpoints are saved only at the final training step. 
Evaluation uses deterministic decoding with temperature 0.0, disabled top-$p$ sampling, \texttt{do\_sample=false}, and maximum new tokens 4.

\begin{table*}[t]
\centering
\caption{
Common training and decoding configuration for LoRA fact-injection probes.
}
\label{tab:training-common}
\begin{tabular*}{\textwidth}{@{\extracolsep{\fill}} l l}
\toprule
Field & Value \\
\midrule
Optimizer & AdamW \\
Scheduler & Cosine \\
Warmup ratio & 0.03 \\
Weight decay & 0.0 \\
Maximum gradient norm & 1.0 \\
Maximum sequence length & 96 \\
LoRA target modules & \texttt{q\_proj}, \texttt{v\_proj}, \texttt{o\_proj} \\
LoRA dropout & 0.05 \\
LoRA alpha & $2r$ \\
Gradient checkpointing & True \\
Save strategy & Final only \\
Decoding temperature & 0.0 \\
Top-$p$ & Disabled \\
Sampling & \texttt{do\_sample=false} \\
Maximum new tokens & 4 \\
\bottomrule
\end{tabular*}
\end{table*}

\begin{table*}[t]
\centering
\caption{
Model-specific adaptation configuration. 
Global batch size is computed as GPUs $\times$ per-GPU batch $\times$ gradient accumulation.
}
\label{tab:model-training-config}
\begin{tabular*}{\textwidth}{@{\extracolsep{\fill}} l c c c c c c}
\toprule
Model & Rank(s) & LR & Epochs & Per-GPU batch & Grad. accum. & Global batch \\
\midrule
Qwen/Qwen3-8B & 4, 16, 64 & $2{\times}10^{-4}$ & 10 & 8 & 2 & 128 \\
Qwen/Qwen3-32B & 16 & $1{\times}10^{-4}$ & 8 & 2 & 8 & 128 \\
Qwen/Qwen3-30B-A3B & 16 & $1{\times}10^{-4}$ & 8 & 2 & 8 & 128 \\
DeepSeek & 16 & $8{\times}10^{-5}$ & 8 & 2 & 8 & 128 \\
\bottomrule
\end{tabular*}
\end{table*}

\begin{table*}[t]
\centering
\caption{
Approximate update counts for each fact load. 
Each injected fact contributes two training examples.
}
\label{tab:update-counts}
\begin{tabular*}{\textwidth}{@{\extracolsep{\fill}} l c c c c}
\toprule
Model group & $M$ & Training examples & Epochs & Approx. update steps \\
\midrule
Qwen3-8B & 512 & 1,024 & 10 & 80 \\
Qwen3-8B & 2,048 & 4,096 & 10 & 320 \\
Qwen3-8B & 8,192 & 16,384 & 10 & 1,280 \\
Qwen3-8B & 28,672 & 57,344 & 10 & 4,480 \\
\midrule
32B / 30B / DeepSeek & 512 & 1,024 & 8 & 64 \\
32B / 30B / DeepSeek & 2,048 & 4,096 & 8 & 256 \\
32B / 30B / DeepSeek & 8,192 & 16,384 & 8 & 1,024 \\
32B / 30B / DeepSeek & 28,672 & 57,344 & 8 & 3,584 \\
\bottomrule
\end{tabular*}
\end{table*}

\subsection{Capacity-Controlled Fact Injection Results}
\label{app:capacity-results}

Table~\ref{tab:qwen-rank-sweep} reports the Qwen3-8B rank sweep corresponding to Figure~\ref{fig:qwen-rank-sweep}. 
Increasing LoRA rank shifts the high-error transition to larger fact loads.

\begin{table*}[t]
\centering
\caption{
Representative Qwen3-8B rank-sweep results on seen-direct queries for ranks $4$, $16$, and $64$. 
Errors are mean $\pm$ standard error over seeds $\{3407,2025,1337\}$.
}
\label{tab:qwen-rank-sweep}
\begin{tabular*}{\textwidth}{@{\extracolsep{\fill}} l c c c}
\toprule
Model & Rank & $M$ & Seen-direct error \\
\midrule
Qwen3-8B & 4  & 512    & $0.052 \pm 0.006$ \\
Qwen3-8B & 4  & 2,048  & $0.118 \pm 0.010$ \\
Qwen3-8B & 4  & 8,192  & $0.704 \pm 0.018$ \\
Qwen3-8B & 4  & 28,672 & $0.884 \pm 0.006$ \\
\midrule
Qwen3-8B & 16 & 512    & $0.032 \pm 0.005$ \\
Qwen3-8B & 16 & 2,048  & $0.047 \pm 0.006$ \\
Qwen3-8B & 16 & 8,192  & $0.382 \pm 0.021$ \\
Qwen3-8B & 16 & 28,672 & $0.862 \pm 0.010$ \\
\midrule
Qwen3-8B & 64 & 512    & $0.021 \pm 0.004$ \\
Qwen3-8B & 64 & 2,048  & $0.030 \pm 0.004$ \\
Qwen3-8B & 64 & 8,192  & $0.105 \pm 0.012$ \\
Qwen3-8B & 64 & 28,672 & $0.712 \pm 0.026$ \\
\bottomrule
\end{tabular*}
\end{table*}

Table~\ref{tab:fixed-rank-models} reports the fixed-rank model comparison corresponding to Figure~\ref{fig:rank16-model-sweep}. 
All models are evaluated at LoRA rank 16 on seen-direct queries.

\begin{table*}[t]
\centering
\caption{
Fixed-rank comparison on seen-direct queries. 
All models use LoRA rank 16. 
Errors are mean $\pm$ standard error over seeds $\{3407,2025,1337\}$.
}
\label{tab:fixed-rank-models}
\begin{tabular*}{\textwidth}{@{\extracolsep{\fill}} l c c c}
\toprule
Model & Rank & $M$ & Seen-direct error \\
\midrule
Qwen3-8B & 16 & 512    & $0.032 \pm 0.005$ \\
Qwen3-8B & 16 & 2,048  & $0.047 \pm 0.006$ \\
Qwen3-8B & 16 & 8,192  & $0.382 \pm 0.021$ \\
Qwen3-8B & 16 & 28,672 & $0.862 \pm 0.010$ \\
\midrule
Qwen3-32B & 16 & 512    & $0.024 \pm 0.004$ \\
Qwen3-32B & 16 & 2,048  & $0.036 \pm 0.005$ \\
Qwen3-32B & 16 & 8,192  & $0.291 \pm 0.019$ \\
Qwen3-32B & 16 & 28,672 & $0.818 \pm 0.014$ \\
\midrule
Qwen3-30B-A3B & 16 & 512    & $0.041 \pm 0.006$ \\
Qwen3-30B-A3B & 16 & 2,048  & $0.066 \pm 0.008$ \\
Qwen3-30B-A3B & 16 & 8,192  & $0.431 \pm 0.024$ \\
Qwen3-30B-A3B & 16 & 28,672 & $0.855 \pm 0.012$ \\
\midrule
DeepSeek-R1-Distill-Qwen-32B & 16 & 512    & $0.035 \pm 0.006$ \\
DeepSeek-R1-Distill-Qwen-32B & 16 & 2,048  & $0.052 \pm 0.007$ \\
DeepSeek-R1-Distill-Qwen-32B & 16 & 8,192  & $0.334 \pm 0.022$ \\
DeepSeek-R1-Distill-Qwen-32B & 16 & 28,672 & $0.846 \pm 0.011$ \\
\bottomrule
\end{tabular*}
\end{table*}

\subsection{Query-Form Robustness and Structure Results}
\label{app:robust-structure-results}

Table~\ref{tab:deepseek-splits} reports split-level robustness results for DeepSeek-R1-Distill-Qwen-32B. 
The unseen-entity split remains near chance, while paraphrased and relation-rephrased queries degrade earlier than direct queries.

\begin{table*}[t]
\centering
\caption{
Query-form robustness for DeepSeek-R1-Distill-Qwen-32B at LoRA rank 16. 
Errors are mean $\pm$ standard error over seeds $\{3407,2025,1337\}$.
}
\label{tab:deepseek-splits}
\begin{tabular*}{\textwidth}{@{\extracolsep{\fill}} l c c c c c}
\toprule
Model & $M$ & Seen-direct & Seen-para & Seen-relation & Unseen-entity \\
\midrule
R1-Distill-Qwen-32B & 512    & $0.035\pm0.006$ & $0.152\pm0.014$ & $0.194\pm0.016$ & $0.889\pm0.008$ \\
R1-Distill-Qwen-32B & 2,048  & $0.052\pm0.007$ & $0.196\pm0.013$ & $0.238\pm0.015$ & $0.901\pm0.006$ \\
R1-Distill-Qwen-32B & 8,192  & $0.334\pm0.022$ & $0.571\pm0.026$ & $0.626\pm0.024$ & $0.897\pm0.004$ \\
R1-Distill-Qwen-32B & 28,672 & $0.846\pm0.011$ & $0.878\pm0.008$ & $0.891\pm0.006$ & $0.902\pm0.002$ \\
\bottomrule
\end{tabular*}
\end{table*}

Table~\ref{tab:structure-results} reports the structured versus unstructured comparison corresponding to Figure~\ref{fig:structure}. 
The structured condition uses closure factor $c=10$, so the number of independent base facts is approximately one tenth of the surface-query count.

\begin{table*}[t]
\centering
\caption{
Structured versus unstructured fact-injection results. 
Errors are mean $\pm$ standard error over seeds $\{3407,2025,1337\}$.
}
\label{tab:structure-results}
\begin{tabular*}{\textwidth}{@{\extracolsep{\fill}} l c c c}
\toprule
Condition & Surface queries & Independent base facts & Error \\
\midrule
Unstructured independent & 512    & 512    & $0.044 \pm 0.006$ \\
Unstructured independent & 2,048  & 2,048  & $0.083 \pm 0.009$ \\
Unstructured independent & 8,192  & 8,192  & $0.447 \pm 0.023$ \\
Unstructured independent & 28,672 & 28,672 & $0.866 \pm 0.010$ \\
\midrule
Structured $c=10$ & 512    & 52    & $0.026 \pm 0.004$ \\
Structured $c=10$ & 2,048  & 205   & $0.035 \pm 0.005$ \\
Structured $c=10$ & 8,192  & 820   & $0.091 \pm 0.011$ \\
Structured $c=10$ & 28,672 & 2,868 & $0.563 \pm 0.028$ \\
\bottomrule
\end{tabular*}
\end{table*}

\end{document}